\pdfoutput=1
\documentclass{article}

\usepackage[preprint, nonatbib]{neurips_2025}

\usepackage[utf8]{inputenc} 
\usepackage[T1]{fontenc}    
\usepackage{hyperref}       
\usepackage{url}            
\usepackage{booktabs}       
\usepackage{amsfonts}       
\usepackage{nicefrac}       
\usepackage{microtype}      
\usepackage{xcolor}

\title{Group Distributionally Robust Machine Learning under Group Level Distributional Uncertainty}

\author{%
\normalfont
Xenia Konti$^{1}$, Yi Shen$^{1}$, Zifan Wang$^{2}$, Karl Henrik Johansson$^{2}$,\\
Michael J. Pencina$^{1}$, Nicoleta J. Economou-Zavlanos$^{1}$, Michael M. Zavlanos$^{1}$ \\
\\
$^{1}$ Duke University \quad
$^{2}$ KTH Royal Institute of Technology \\
}

\usepackage{graphicx}
\graphicspath{{figs/}}
\usepackage{amsmath,amssymb}
\usepackage{amsthm}
\newtheorem{assumption}{Assumption}[section]
    \newtheorem{example}{Example}[section]
\newtheorem{property}{Property}[section]
\newtheorem{lemma}{Lemma}[section]
\newtheorem{theorem}{Theorem}[section]
\usepackage{subcaption} 
\usepackage{booktabs,multirow,tabularx,caption}
\usepackage{bbm}
\usepackage{tikz-cd}
\usepackage{stmaryrd}
\usepackage{xcolor}
\usepackage{csquotes}
\usepackage{tikz}
\usepackage{caption}
\usepackage{subcaption}
\usepackage{xcolor}
\usepackage{booktabs}
\usepackage{microtype}
\usepackage{algorithm}
\usepackage{algpseudocode}
\usepackage{subcaption} 
\usepackage{comment}
\newtheorem{proposition}{Proposition}[section]
\newtheorem{definition}{Definition}
\usepackage{booktabs}
\usepackage{multirow}
\usepackage[table]{xcolor}
\usepackage{colortbl}
\usepackage{caption}
\usepackage{longtable}
\usepackage{adjustbox}
\usepackage{lineno}
\usepackage{bbm}
\definecolor{lightblue}{rgb}{0.8, 0.9, 1.0}
\definecolor{midblue}{rgb}{0.6, 0.8, 1.0}
\definecolor{strongblue}{rgb}{0.3, 0.6, 1.0}
\definecolor{lightred}{rgb}{1.0, 0.8, 0.8}
\definecolor{midred}{rgb}{1.0, 0.6, 0.6}

\usepackage{physics}

\begin{document}
\maketitle
\begin{abstract}
   The performance of machine learning (ML) models critically depends on the quality and representativeness of the training data. In applications with multiple heterogeneous data generating sources, standard ML methods often learn spurious correlations that perform well on average but degrade performance for atypical or underrepresented groups. Prior work addresses this issue by optimizing the worst-group performance. However, these approaches typically assume that the underlying data distributions for each group can be accurately estimated using the training data, a condition that is frequently violated in noisy, non-stationary, and evolving environments. In this work, we propose a novel framework that relies on Wasserstein-based distributionally robust optimization (DRO) to account for the distributional uncertainty within each group, while simultaneously preserving the objective of improving the worst-group performance. We develop a gradient descent-ascent algorithm to solve the proposed DRO problem and provide convergence results. Finally, we validate the effectiveness of our method on real-world data.
   
\end{abstract}
\section{Introduction}
 Machine learning models are typically trained to minimize the average loss over training datasets, under the assumption that both training and testing samples are drawn independently from the same distribution. However, in real-world applications, this assumption is often violated. For instance, data can be generated from multiple heterogeneous environments, such as different hospitals, geographic regions, or demographic groups, and each of these environments can be associated with a distinct data distribution. In addition, even within a single environment, the data distribution may shift over time due to factors like temporal drift, changes in population demographics, or finite sampling bias. In these real-world applications, models trained without accounting for data heterogeneity or distribution shifts may show disparate performance across different subpopulations in the dataset -- even if they achieve low average loss over the whole population \cite{duchi2019distributionally} -- or may even show average performance degradation when transferred from a training to a test set within the same environment.
These shortcomings can be especially problematic in high-stakes domains like healthcare \cite{seyyed2020chexclusion} and finance \cite{fuster2022predictably, khandani2010consumer}, where models should perform equally well across different population subgroups and maintain their performance in the presence of distribution shifts that can occur when they are deployed on environments that are different from those they were trained on.

A principled framework that has been widely used to introduce robustness to possible distribution shifts between training and test environments is Distributionally Robust Optimization (DRO). DRO employs a set of multiple plausible distributions that may describe future test environments, known as the ambiguity set, and formulates the robust learning problem as a min-sup problem that returns a model that minimizes the worst-case loss over this ambiguity set \cite{goh2010distributionally}, as shown in Figure \ref{fig:dro}. In classical DRO, the ambiguity set is typically constructed as a ball around the data generating distribution of the training set, with radius defined by different divergence measures such as $f$-divergence or Wasserstein distance \cite{namkoong2016stochastic, kuhn2019wasserstein, chen2018distributionally}. While this formulation captures uncertainty around a single environment, it does not capture the effect of multi-source data that are common in practice.

To handle data generated from multiple environments with different data generating distributions, Group DRO (GDRO) methods have been proposed \cite{sagawa2019distributionally, oren2019distributionally}. GDRO methods typically define an ambiguity set that contains all linear mixtures of those group distributions -- Figure \ref{fig:group} -- and formulate a min-max optimization problem that returns a model that performs well on the combination of environments with the worst expected loss. While this problem has been solved under perfect knowledge of the environments and their corresponding distributions using stochastic gradient update methods \cite{sagawa2019distributionally, zhang2023stochastic, soma2022optimal, yu2024efficient}, there is limited work that considers uncertainty in the training data generating distributions. In this direction, the work in \cite{ghosal2023distributionally} assumes that the training data labels in each group are uncertain and proposes a probabilistic group membership approach to address this challenge. However, addressing distribution shifts between the training and test sets in the local environments in a Group DRO setting is a problem that, to the best of our knowledge, still remains unexplored. 

\begin{figure}[t]
    \centering

    \begin{subfigure}[t]{0.3\textwidth}
        \includegraphics[width=\linewidth]{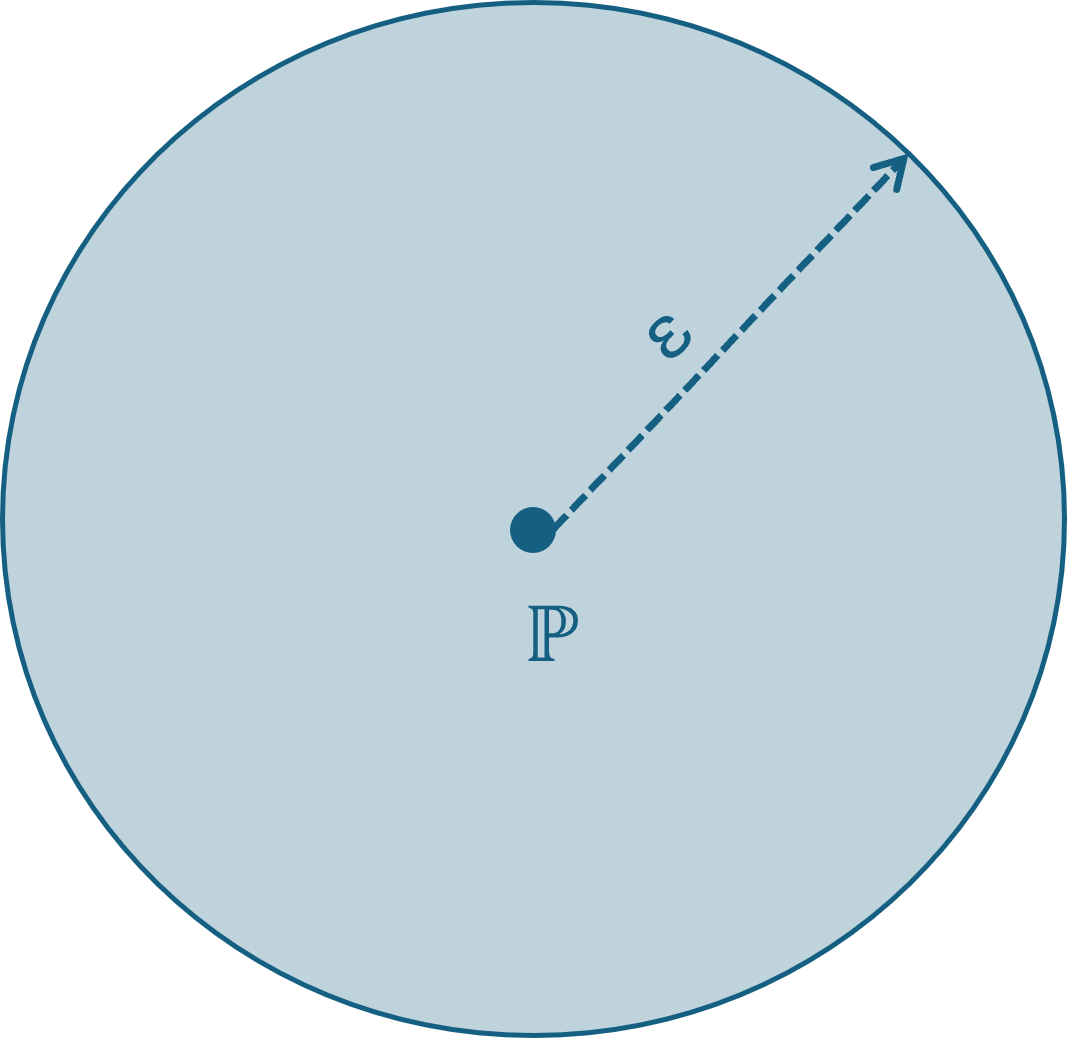}
        \caption{}
        \label{fig:dro}
    \end{subfigure}
    \hfill
    \begin{subfigure}[t]{0.3\textwidth}
        \includegraphics[width=\linewidth]{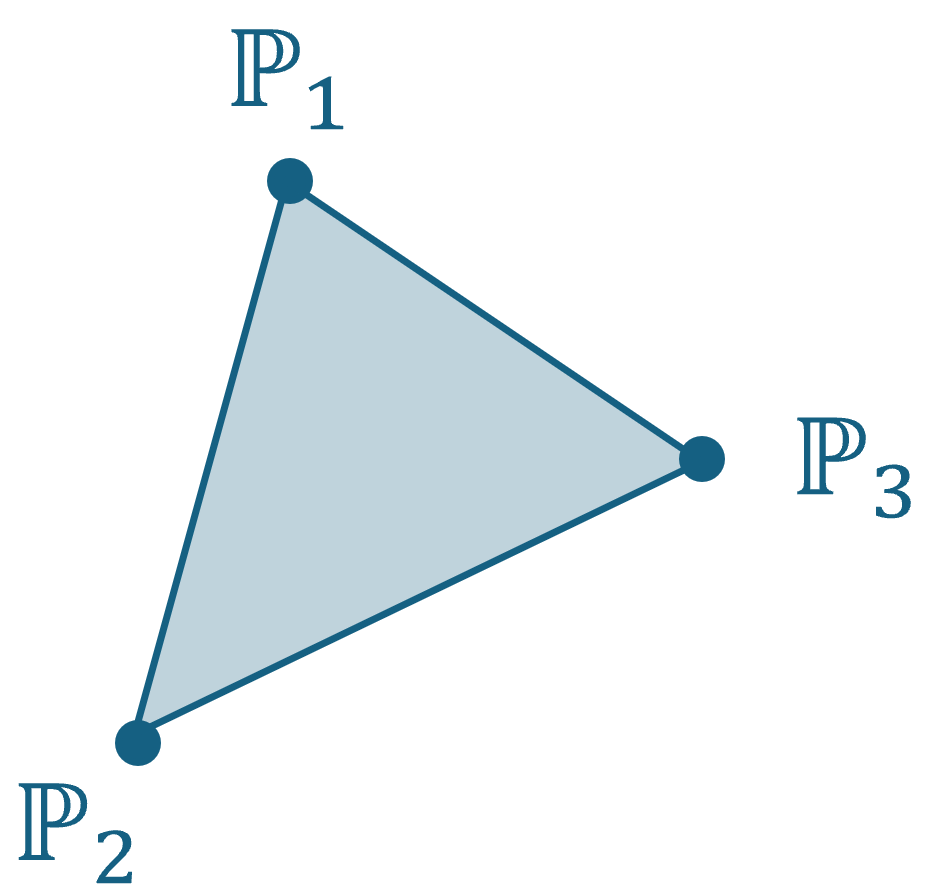}
        \caption{}
        \label{fig:group}
    \end{subfigure}
    \hfill
    \begin{subfigure}[t]{0.3\textwidth}
        \includegraphics[width=\linewidth]{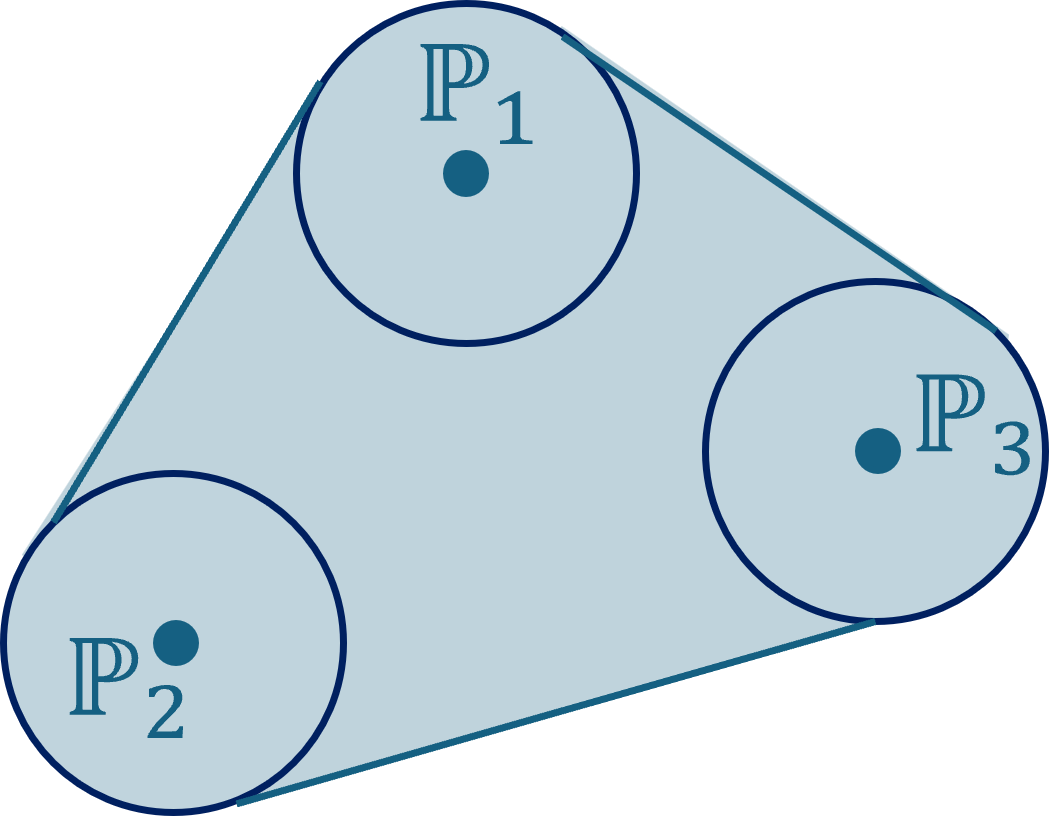}
        \caption{}
        \label{fig:group_plus_dro}
    \end{subfigure}

    \caption{DRO ambiguity sets: (a) Ambiguity set in classical DRO that contains all distributions that are $\varepsilon$ away from the data generating distribution $\mathbb{P}$ that generated the training set; (b) Ambiguity set in classical Group DRO that contains all distributions that lie in the simplex formed by the data generating distributions $\{\mathbb{P}_1, \mathbb{P}_2, \mathbb{P}_3,\cdots\}$ of a collection of environments; and (c) Ambiguity set of our method that contains all distributions that are $\varepsilon_i$ away from the data generating distribution $\mathbb{P}_i$ of each environment $i=1,2,3,\cdots$, as well as all mixtures of those.}
    \label{fig:comparison}
\end{figure}

In this work, we propose a framework that addresses both heterogeneity across the data-generating environments and distributional shifts within each of these environments. We formulate the robust learning problem as a nested min-max-sup optimization: the outer minimization seeks model parameters that minimize the expected loss, the first-layer maximization optimizes over group weights to identify the worst-performing subpopulation, and the innermost maximization identifies the worst-case distribution within each group’s local ambiguity set. This three-level structure presents substantial algorithmic and analytical challenges. In particular, the inner supremum over distributions requires approximating adversarial perturbations, which introduces computational complexity not present in classical Group DRO. Moreover, this nested structure complicates the theoretical analysis: convergence guarantees require careful coordination of three interacting update steps (distributional perturbation, group reweighting, and parameter updates), each with their own sensitivity to smoothness and step size conditions. To address these challenges, we design a tractable three-step gradient algorithm and provide convergence results under standard assumptions. Our approach thus fills a critical gap in the robust optimization literature by simultaneously addressing across-group and within-group uncertainties in a unified and provably convergent framework.We validate our approach on real-world datasets and demonstrate its advantages over classical robust methods.

The remainder of the paper is organized as follows. Section~\ref{sec:setup} introduces the problem setup and presents our formulation for group-level distributional uncertainty. Section~\ref{sec:methodology} describes the proposed methodology and algorithm design. Section~\ref{sec:experiments} presents our experimental setup and empirical results on two real-world datasets. Finally, Section~\ref{sec:conclusion} concludes the paper.


\section{Problem Setup}
\label{sec:setup}
We consider a training set $D_{train}$ consisting of data points sampled from $G$ environments with independent generating distributions $\{\mathbb{P}_{X,Y}^1, \mathbb{P}_{X,Y}^2, \cdots, \mathbb{P}_{X,Y}^G\}$, where $X \in \mathcal{X}$ and $Y \in \mathcal{Y}$ are random variables denoting covariates and outcomes, respectively. We assume that each data point in $D_{train}$ is of the form $\{x_i, y_i, g_i\}$ for $i=1,\cdots, N$, where $x_i \in \mathcal{X}$ denotes the covariates, $y_i \in \mathcal{Y}$ denotes the outcomes, $g_i \in \{1, \cdots, G\}$ denotes the environment from which the point was sampled, and $N$ is the size of the dataset. We also assume that in the training set, a data-point is generated from an environment $g$ with probability $p_g$, such that $\sum_{g=1}^G p_g = 1$. Then, the data generating distribution $\mathbb{P}_{X,Y}$ associated with the training set $D_{train}$ can be defined as a mixture of the environmental distributions as
$$\mathbb{P}_{X,Y} = \sum_{g=1}^G p_g \cdot  \mathbb{P}_{X,Y}^g.$$

Using the data generated from $\mathbb{P}_{X,Y}$, typical machine learning techniques rely on Empirical Risk Minimization (ERM) to compute a parametric model $f_\theta: \mathcal{X} \rightarrow \mathcal{Y}$, where $\theta$ are the model parameters, that minimizes the expected loss $\mathcal{L}$ over the training set, i.e.,
\begin{equation}
    \label{eqn: erm}
    \min_{\theta} \mathbb{E}_{(x,y) \sim \mathbb{P}_{X, Y}}[\mathcal{L}(f_\theta; x, y)].
\end{equation}

\subsection{Distribution Shifts Across Groups}\label{sec:group_DRO}

When data are generated from multiple possibly heterogeneous environments, training a model on the distribution $\mathbb{P}_{X,Y}$ using \eqref{eqn: erm}, performs well on average but can lead to disparate performance across the individual environments \cite{hong2023predictive, sagawa2019distributionally}. In these situations, Group DRO can be used to improve the model performance across the different environments.

Group DRO \cite{hu2018does,oren2019distributionally,sagawa2019distributionally} takes into account the data generating distributions of the environments $\{\mathbb{P}_{X,Y}^1, \cdots, \mathbb{P}_{X,Y}^G\}$ and learns a model that performs best for the worst-case distribution among the groups. To this end, let $$\mathcal{Q}:= \left\{\sum_{g = 1}^G q_g \mathbb{P}_{X,Y}^g: q \in \Delta_G\right\}$$ denote an ambiguity set consisting of all linear combinations of the individual group distributions, where $q=[q_1, \cdots, q_G]$ denotes the vector of linear weights and $\Delta_G$ denotes the $G$-dimensional probability simplex. Then, the goal of Group DRO is to learn a model $f_\theta$ that optimizes the worst-case expected loss over the ambiguity set $\mathcal{Q}$, i.e.,
\begin{equation}
    \label{eqn:group-dro1}
    \min_{\theta \in \Theta} \sup_{Q \in \mathcal{Q}} \mathbb{E}_{(x,y) \sim Q} [\mathcal{L}(f_\theta; x,y)].
\end{equation}

Since each distribution $Q \in \mathcal{Q}$ is a weighted combination of the local group distributions $\mathbb{P}_{X,Y}^g$, Group DRO effectively searches over the space of possible weighted combinations of the local group distributions to emphasize those with higher loss. Therefore, instead of biasing the model toward environments that dominate the training set, as in the case of ERM, Group DRO encourages the model to improve performance on the most challenging or under-performing environments.

Finally, since each distribution $Q \in \mathcal{Q}$ is a weighted combination of the local group distributions $\mathbb{P}_{X,Y}^g$, the expected loss over $Q$ can be written as
$$\mathbb{E}_{(x,y) \sim Q} [\mathcal{L}(f_\theta; x,y)]=  \sum_{g = 1}^G q_g \mathbb{E}_{(x,y)\sim \mathbb{P}_g}[\mathcal{L}(f_\theta; x,y)] = \sum_{g=1}^G q_g \mathcal{L}_g(f_\theta),$$ 
where $\mathcal{L}_g(f_\theta) := \mathbb{E}_{(x,y) \sim \mathbb{P}_g}[\mathcal{L}(f_\theta; x, y)]$ denotes the expected group-level loss. Substituting this expectation into \eqref{eqn:group-dro1}, we obtain an equivalent min-max formulation of the Group DRO problem as
\begin{equation}
\label{eqn:group-dro2}
\min_{\theta \in \Theta} \max_{q \in \Delta_G} \sum_{g = 1}^G q_g  \mathcal{L}_g(f_\theta).
\end{equation}

\subsection{Distribution Shifts Across and Within Groups}

Different to existing literature, in this paper we assume that the data generating distribution $\mathbb{P}
_{X,Y}^g$ in each environment is unknown. Even though the empirical distribution $\hat{\mathbb{P}
}_{X,Y}^g$ of each group $g \in \{1, \cdots, G\}$ can be estimated from training data and used to approximate the true data generating distribution $\mathbb{P}
_{X,Y}^g$, this approximation can contain errors due to finite-sampling bias or possible distribution shifts.  

To address uncertainty in the data generating distributions in the local environments, we extend the group DRO framework discussed in Section~\ref{sec:group_DRO} by combining it with a local DRO objective at each local environment. The goal is to learn models that are robust to both changes in the mixture of the local environments as well as to distribution shifts within each of the local environments. Specifically, given a distance metric $D$ between distributions, let $$\mathcal{P}_g = \{\mathbb{P}: D(\mathbb{P}, \hat{\mathbb{P}}_{X,Y}^g) \leq \epsilon_g\}$$ denote an ambiguity set containing all possible data generating distributions for group $g\in \{1,\cdots,G\}$, that is a ball of radius $\epsilon_g >0$ around the empirical distribution $\hat{\mathbb{P}}_{X,Y}^g$. As a distance metric between two distributions we use the 1-Wasserstein metric defined below.

\begin{definition}[1-Wasserstein Distance {\cite{villani2009optimal}}]
Let $\mathbb{P}$ and $\mathbb{P}'$ be two probability distributions on a Polish space $\Xi \subseteq \mathbb{R}^d$ with finite second moments. Let $\Gamma(\mathbb{P}, \mathbb{P}')$ denote the set of all couplings (i.e., joint distributions) with marginals $\mathbb{P}$ and $\mathbb{P}'$ and let $c: \Xi \times \Xi \rightarrow [0, \infty)$ denote the transportation cost. Then, the 1-Wasserstein distance between $\mathbb{P}$ and $\mathbb{P}'$ is defined as
\[
W_1(\mathbb{P}, \mathbb{P}') :=  \inf_{\gamma \in \Gamma(\mathbb{P}, \mathbb{P}')} \int_{\Xi \times \Xi} c(x, y)\, d\gamma(x, y) .
\]
\end{definition}

Given the uncertainty set $\mathcal{P}_g$ for each environment $g \in \{1, \cdots, G\}$, we can define the robust group-level loss as
\begin{equation}
    \label{eqn: rob_env_loss}
    \mathcal{L}_g^{ROB}(f_\theta) = \sup_{\mathbb{P}_{X,Y}^g \in \mathcal{P}_g}\mathbb{E}_{(x,y) \sim \mathbb{P}_{X,Y}^g} [\mathcal{L}(f_\theta; x,y)]. 
\end{equation}

Substituting the robust group-level loss into \eqref{eqn:group-dro2}, we obtain the proposed group DRO under group-level distributional uncertainty problem with the objective to learn a model $f_{\theta}$ that is robust to both the worst-case mixture of environments and to distribution shifts within environments, i.e.,
\begin{equation}
    \label{eqn: fina_rob}
    \min_{\theta \in \Theta} \max_{q \in \Delta_G} \sum_{g = 1}^G q_g \mathcal{L}_g^{ROB}(f_\theta).    
\end{equation}

\section{Methodology}
\label{sec:methodology}
In this section, we introduce a gradient method to solve the min-max Group DRO with group-level distributional uncertainty problem \eqref{eqn: fina_rob}. In particular, we design an iterative algorithm that at each iteration first computes the robust loss $\mathcal{L}_g^{ROB}$ for each environment $g\in\{1,\cdots,G\}$, and then performs gradient descent mirror ascent steps to update the parameters of the min-max Group DRO problem. In the following subsections we analyze these two parts of our algorithm. 

\subsection{Lagrangian Relaxation of the Robust Group-Level Loss}

\begin{algorithm}[t]
    \caption{Gradient Ascent for Worst-case Perturbations}
    \label{alg: dro_ga}
    \begin{algorithmic}[1]
    \Require $(x, y)$, $\eta_z$, $T_{rob}$, cost function $c: (\mathcal{X} \times \mathcal{Y})\times (\mathcal{X} \times \mathcal{Y}) \rightarrow \mathbb{R}_+$ and model parameters $f_\theta$
    \State Initialize $z^0$ with $(x, y)$
    \For{$t = 1$ to ($T_{rob}$)}
        \State $\phi(f_\theta; (x, y), z^{t-1}) = \mathcal{L}(f_\theta; z^{t-1}) - \gamma c((x, y), z^{t-1})$ 
        \State $z^t \leftarrow z^{t-1} + \eta_z \nabla_{z}\phi(f_\theta; (x, y), z^{t-1})$
    \EndFor
    \State Return $z^{T_{rob}}$
    \end{algorithmic}
\end{algorithm}
Directly computing the robust group-level loss $\mathcal{L}_g^{ROB}$ for every group $g\in\{1,\cdots,G\}$ is generally intractable, since it requires computing the supremum over an infinite ambiguity set of distributions $\mathcal{P}_g$. When the model $f_\theta$ is convex with respect to the model parameters $\theta$, e.g., linear \cite{chen2018robust} or a logistic regression \cite{shafieezadeh2015distributionally}, the dual formulation of the robust loss in \eqref{eqn: rob_env_loss} can be cast as a convex optimization problem that can be efficiently solved using existing solvers.

However, when the model $f_\theta$ is not convex with respect to the model parameters $\theta$ as, e.g., in the case of Neural Networks, to the best of our knowledge, tractable dual formulations of the robust loss \eqref{eqn: rob_env_loss} do not exist. In this case, to compute the robust loss we instead resort to its Lagrangian relaxation
\begin{equation}
    \label{eqn: langr_loss}
    \mathcal{L}_{g,\gamma}^{ROB}(f_\theta) = \sup_{\mathbb{P}_{X,Y}^g \in \mathcal{P}_g} \biggl[\mathbb{E}_{(x,y) \sim \mathbb{P}_{X,Y}^g} [\mathcal{L}(f_\theta; x,y)] - \gamma W_1(\mathbb{P}_g, \hat{\mathbb{P}_g})\biggr],
\end{equation}
where $\gamma \geq 0$ is a fixed penalty parameter. The Lagrangian relaxation of the robust loss in \eqref{eqn: langr_loss} can be efficiently computed as shown in the following result.

\begin{proposition} \label{prop:robust_loss}(Proposition 1 in \cite{sinha2017certifying})
    Let $\mathcal{L}: \Theta \times (\mathcal{X}, \mathcal{Y}) \rightarrow \mathbb{R}$ be a loss function and $c: (\mathcal{X}, \mathcal{Y}) \times (\mathcal{X}, \mathcal{Y}) \rightarrow \mathbb{R}_+$ be a continuous transportation cost function. Then, for any distribution $\hat{\mathbb{P}}_{X,Y}^g$, $\gamma \geq 0$, and any uncertainty set $\mathcal{P} = \{\mathbb{P}: W_1(\mathbb{P}, \hat{\mathbb{P}}_{X,Y}^g) \leq \epsilon\}$ we have 
    \begin{equation}
        \label{eqn: langr_loss_relaxed}
        \mathcal{L}_{g, \gamma}^{ROB}(f_\theta)= \mathbb{E}_{(x,y) \sim \hat{\mathbb{P}}_{X,Y}^g} \biggl[\sup_{(x^\prime, y^\prime) \in \mathcal{X} \times \mathcal{Y}} \phi(f_\theta; (x, y), (x^\prime, y^\prime))\biggr],
    \end{equation}
     where $\phi(f_\theta; (x, y), (x^\prime, y^\prime)) = \mathcal{L}(f_\theta; x^\prime,y^\prime) - \gamma c((x,y), (x^\prime, y^\prime))$ is a penalized loss. 
\end{proposition}

Proposition~\ref{prop:robust_loss} allows to compute the robust loss $\mathcal{L}_{g, \gamma}^{ROB}$ of each group $g\in\{1,\cdots,G\}$ by exhaustively searching over the support $\mathcal{X} \times \mathcal{Y}$ for points $(x^\prime, y^\prime)$ that maximize the penalized loss $\phi(f_\theta; (x, y), (x^\prime, y^\prime))$. However, when the support $\mathcal{X} \times \mathcal{Y}$ is large or infinite -- e.g., for continuous variables -- searching over the support can be computationally intractable. In this case, similar to the work in \cite{sinha2017certifying}, we propose a gradient ascent algorithm to instead approximate the robust loss $\mathcal{L}_{g, \gamma}^{ROB}$. The proposed algorithm is summarized in Algorithm~\ref{alg: dro_ga}. Specifically, for every sample $(x,y)\in D_{train}$ in the training set, Algorithm~\ref{alg: dro_ga} iteratively updates the point $z=(x^\prime, y^\prime)\in \mathcal{X} \times \mathcal{Y}$ to maximize the penalized loss $\phi(f_\theta; (x, y), (x^\prime, y^\prime))$.

\subsection{Solution of the Group DRO with Group-Level Distributional Uncertainty Problem}

\begin{algorithm}[t]
\caption{Group DRO with distributional uncertainty per group}
\label{alg: minimaxfairauc}
\begin{algorithmic}[1]
\Require Training set $D_{train} = \{x_i, y_i, g_i\}_{i=1}^N$, model $f_\theta$, number of iterations $T$, learning rates $\{ \eta_\theta, \eta_q\}$, cost parameter $\gamma$, cost function $c: (\mathcal{X} \times \mathcal{Y})\times (\mathcal{X} \times \mathcal{Y}) \rightarrow \mathbb{R}_+$
\State Initialize $\theta_0 \in \Theta$ randomly and  $q_g^0 = \frac{N_g}{N} \text{ }\forall g \in \{1, \cdots, G\}$
\For{$t = 1$ to $T$}
    \State $B_t = []$
    \For{$(x_i, y_i, g_i) \in D_{train}$}
        \State Use Algorithm \ref{alg: dro_ga} to obtain $z_i = arg\max_{z \in (\mathcal{X}\times \mathcal{Y})} \phi(f_{\theta_{t-1}}; (x_i, y_i), z)$
    \EndFor
    \For{$g \in \{1, \cdots, G\}$}
        \State $\mathcal{L}_{g, \gamma}^{ROB} (f_{\theta_{t-1}}) = \frac{1}{N_g}\sum_{i=1}^{N_g} \phi(f_{\theta_{t-1}}; (x_i, y_i), z_i)$
        \State $\nabla \mathcal{L}_{g, \gamma}^{ROB} (f_{\theta_{t-1}}) = \frac{1}{N_g}\sum_{i=1}^{N_g} \pdv{\phi(f_{\theta_{t-1}}; (x_i, y_i), z_i)}{\theta}$ 
    \EndFor
    \State $m_g^t \gets q_g^{t-1} \exp\biggl\{\eta_q\mathcal{L}_{g, \gamma}^{ROB} (f_{\theta_{t-1}}) \biggr\} \text{  } \forall g \in \{1, \cdots, G\}$
    \State $q_g^t \gets \frac{m_g^t}{\sum_{g=1}^G m_g^t} \text{  } \forall g \in \{1, \cdots, G\}$     
    \State $\theta_t \gets \theta_{t-1} - \eta_\theta  \sum_{g=1}^G q_g^{t} \nabla \mathcal{L}_{g, \gamma}^{ROB} (f_{\theta_{t-1}})$
    
\EndFor
\end{algorithmic}
\end{algorithm}

In this section we present an algorithm to solve the min-max group DRO with group-level distributional uncertainty problem \eqref{eqn: fina_rob}.
The proposed algorithm is summarized in Algorithm \ref{alg: minimaxfairauc}.
Specifically, Algorithm \ref{alg: minimaxfairauc} is an iterative method that consists of the following steps. First, for each point in the training dataset $(x_i,y_i, g_i)\in D_{train}$, Algorithm \ref{alg: dro_ga} is used to compute the point $z_i$ that, given the current model parameters $\theta_{t-1}$, maximizes the penalized loss $\phi(f_{\theta_{t-1}}; (x_i, y_i), z_i)$, as shown in lines 4-6. Then, using \eqref{eqn: langr_loss_relaxed} and the point $z_i$, Algorithm~\ref{alg: minimaxfairauc} computes the approximate expected robust loss $\mathcal{L}_{g, \gamma} ^{ROB}(f_{\theta_{t-1}})$ and its gradient with respect to $\theta$, $\nabla\mathcal{L}_{g, \gamma} ^{ROB}(f_{\theta_{t-1}})$, for the current model $\theta_{t-1}$ and each environment $g \in \{1, \cdots, G\}$, as shown in lines 7-10. Next, using the robust losses computed in line 8, the weights $q_g^t$ for each environment $g \in \{1, \cdots G\}$ are updated by a mirror ascent step, as shown in lines 11-12. Finally, given the updated weights $q_g^t$ and the robust gradients for each environment, Algorithm~\ref{alg: minimaxfairauc} updates the model parameters $\theta_t$ by a gradient descent step, as shown in line 13.
The Convergence Analysis of the proposed Algorithm is provided in the Appendix.
\section{Numerical Experiments}
\label{sec:experiments}

In this section, we evaluate the performance of our proposed method against several baseline approaches on a real-world datasets from the field of finance. Our framework is designed to address both heterogeneity across different subpopulations in the data and possible distribution shifts within individual subpopulations between the training and test sets. Through our experiments, we aim to demonstrate that failing to account for either of these two objectives can lead to sub-optimal performance when both are existent in the data. To this end, we compare against three baseline methods: (1) Empirical Risk Minimization (ERM) that trains a model to minimize the average loss over the entire training set without accounting for group identities or distribution uncertainties, (2) Distributionally Robust Optimization (DRO) as proposed in \cite{sinha2017certifying} that accounts for distributional shifts between the training and test sets but ignores group structure, and (3) Group DRO (GDRO) as proposed in \cite{sagawa2019distributionally} that explicitly addresses group-wise performance disparities but assumes full knowledge of each group’s data generating distribution without modeling within-group distributional uncertainty.

Our method, as well as the baseline approaches, is model-agnostic. For evaluation, we instantiate the model function $f_\theta$ as a simple feedforward neural network composed of two fully connected layers. The first layer maps the input features to a 64-dimensional hidden representation, followed by an Exponential Linear Unit (ELU) activation. This is succeeded by a second hidden layer with 32 units and another ELU activation. The final output is produced by a linear layer projecting the hidden representation to a single scalar value. We use the same model architecture across all our following experiments. 

We evaluate all methods on test sets that have not been observed during training. To assess the effectiveness of each method, we focus on their accuracy. For each group $g \in \{1, \cdots, G\}$ with a test dataset $D_{test}^g = \{(x_i^g, y_i^g)\}_{i=1}^{N_g}$, we define the accuracy of a model $f_\theta$ as 
$$Acc_g = \sum_{i=1}^{N_g}\frac{\mathbbm{1}\{y_i^g == f_\theta(x_i^g)\}}{N_g} 100\%$$

Given the accuracy for each group, we report three evaluation metrics:\\ 
(i) the average accuracy across all groups $$\text{Average Accuracy} := \frac{\sum_{g=1}^G Acc_g}{G},$$ (ii) the range of accuracies among groups $$\text{Range of Accuracy} := \max_{g \in \{1, \cdots, G\}} Acc_g - \min_{g \in \{1, \cdots, G\}} Acc_g,$$ and (iii) the worst-case accuracy observed across groups $$\text{Worst-Case Accuracy}:= \min_{g \in \{1, \cdots, G\}} Acc_g$$

Consistent with the evaluation protocol in \cite{sagawa2019distributionally}, we report both the average accuracy and the worst-case group accuracy to assess overall model performance and its ability to generalize to the most under-performing group. In addition to these metrics, we include the accuracy range across groups to quantify performance disparities, thereby providing a complementary measure of fairness and consistency in group-level outcomes. The higher the average and worst-case accuracy and the lower the range of accuracy, the better a model's performance.

 Finally, training and evaluation were executed on a Linux workstation with two NVIDIA TITAN RTX cards (24\,GB each, driver 545.23.08, CUDA 12.3).
\subsection{Distributional Shift due to Environmental shifts}\label{sec:Income_prediction}
\begin{figure}[t]
  \centering
  \includegraphics[width=\linewidth]{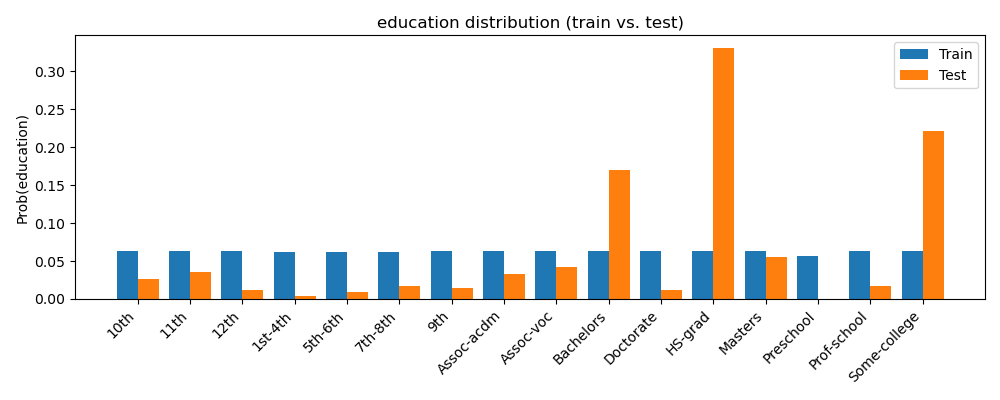}
  \caption{Train–test shift in the \texttt{education} marginal on the Adult dataset.
  We construct training splits with a uniform distribution over \texttt{education}, while the test split retains the dataset’s natural distribution. The example shown is for seed $42$; other seeds (\,$18$, $1999$, $2025$, etc.) realize the same pattern by construction.}
  \label{fig:education-shift}
\end{figure}

We first evaluate on the Adult Income dataset \cite{adult_2}, a widely used benchmark of 47{,}621 individuals with 15 demographic and occupational features (e.g., age, race, gender, education level, marital status, occupation) and a binary label indicating whether annual income exceeds $\$50{,}000$.

The dataset is frequently used in robustness studies because of pronounced group heterogeneity. Prior work (e.g., \cite{soma2022optimal, zhang2023stochastic}) typically defines groups using sensitive attributes such as race and gender. Following \cite{sagawa2019distributionally}, we adopt six intersectional groups based on race $\{\text{White}, \text{Black}, \text{Other}\}$ crossed with the income label $\{\leq $50\text{K}, > $50\text{K}\}$:
\begin{itemize}
  \item \textbf{Group 0} — \emph{White, income $>\!50$K}: 10{,}485 individuals ($\approx 22\%$).
  \item \textbf{Group 1} — \emph{White, income $\le\!50$K}: 30{,}301 individuals ($\approx 64\%$).
  \item \textbf{Group 2} — \emph{Black, income $>\!50$K}: 555 individuals ($\approx 1\%$).
  \item \textbf{Group 3} — \emph{Black, income $\le\!50$K}: 3{,}980 individuals ($\approx 8\%$).
  \item \textbf{Group 4} — \emph{Other race, income $>\!50$K}: 501 individuals ($\approx 1\%$).
  \item \textbf{Group 5} — \emph{Other race, income $\le\!50$K}: 1{,}799 individuals ($\approx 4\%$).
\end{itemize}

We run ten independent trials with fixed seeds $\{42, 18, 2025, 1999, 1453, 1821, 2023, 2024, 2020, 2021\}$. Our goal is twofold: first, to compare our method against baselines for robustness across the above groups; second, to assess robustness under a train–test distribution shift. To induce a controlled covariate shift, we construct training splits with a \emph{uniform} marginal over \texttt{education} and evaluate on test splits that follow the original Adult distribution. The resulting shift is illustrated in Figure~\ref{fig:education-shift}.
\newpage
\begin{table}[H]
  \centering
  \small
  \setlength{\tabcolsep}{6pt} 
  \renewcommand{\arraystretch}{1.15} 

  \caption{Evaluation results of the four methods for the Income prediction task. 
  Best is strong blue, runner-up light blue, worst light red.}
  \label{tab:seed42_results}

  \begin{tabularx}{\linewidth}{ll|*{4}{>{\centering\arraybackslash}X}}
    \toprule
    \multicolumn{2}{c|}{} & \multicolumn{4}{c}{\textbf{Task C}}\\
    \multicolumn{2}{c|}{} & ERM & DRO ($\gamma{=}9$) & Group & Ours ($\gamma{=}10^{-4}$)\\
    \midrule
    \multirow{3}{*}{}
    & Average Accuracy
      & \cellcolor{lightred}$0.5471 \pm 0.0100$
      & \cellcolor{lightred}$0.5165 \pm 0.0037$
      & \cellcolor{lightblue}$0.6953 \pm 0.0269$
      & \cellcolor{strongblue}$0.7148 \pm 0.0105$\\
    & Worst-Group Accuracy
      & \cellcolor{lightred}$0.0815 \pm 0.0220$
      & \cellcolor{lightred}$0.0287 \pm 0.0073$
      & \cellcolor{lightblue}$0.5607 \pm 0.0388$
      & \cellcolor{strongblue}$0.6126 \pm 0.0605$\\
    & Accuracy Range
      & \cellcolor{lightred}$0.9154 \pm 0.0241$
      & \cellcolor{lightred}$0.9709 \pm 0.0077$
      & \cellcolor{lightblue}$0.2571 \pm 0.0789$
      & \cellcolor{strongblue}$0.1934 \pm 0.0927$\\
    \bottomrule
  \end{tabularx}
\end{table}

For the robustness parameter $\gamma$, we sweep $\gamma \in \{10^{-4}, 10^{-3}, 10^{-2}, 10^{-1}, 0.5, 1, 3, 5, 6, 7, 8\}$ to select the best value via fine-tuning and to study sensitivity. The remaining hyperparameters in Algorithm~\ref{alg: dro_ga} are
$\eta_{\theta}=0.1$, $\eta_{q}=0.1$, $\eta_{z}=0.05$, $T_{\text{rob}}=100$, and $T=200$.

Numerical results, summarized in Table \ref{tab:seed42_results}, reveal three clear trends. First, both ERM and standard DRO perform very poorly: ERM achieves only $0.55$ average accuracy with worst-group accuracy just $0.08$, and both methods exhibit extreme disparity across groups, with accuracy ranges exceeding $0.9$. Second, Group DRO substantially improves fairness by boosting worst-group accuracy to $0.56$ and reducing the accuracy gap to $0.26$. Third, our method achieves the best trade-off across all metrics: it reaches the highest mean accuracy ($0.715 \pm 0.011$), the highest worst-group accuracy ($0.613 \pm 0.061$), and the lowest disparity ($0.193 \pm 0.093$). 

Figure~\ref{fig:adult_gamma} further illustrates how performance evolves as $\gamma$ is varied. Our method is consistently stable across the range of values, with average accuracy remaining above $0.70$, worst-group accuracy above $0.58$, and range below $0.22$. By contrast, standard DRO models collapse across all $\gamma$, with worst-group accuracies near~0 and accuracy ranges close to~1. Together, these results confirm that our approach not only outperforms all baselines numerically, but also remains robust and reliable across hyperparameter choices.
\begin{figure}[H]
    \centering
    \includegraphics[width=\linewidth]{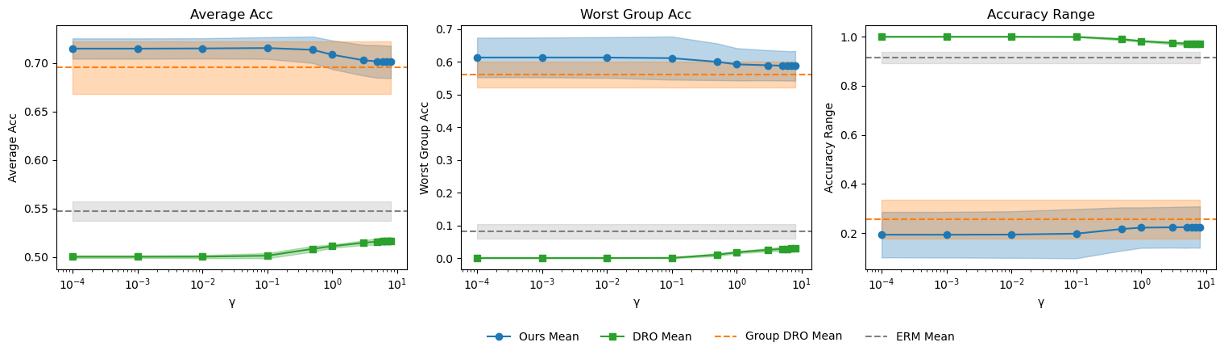}
    \caption{Performance of ERM, DRO, Group DRO, and our method on the \texttt{Adult} dataset under distribution shift on the \textit{education} attribute as a function of $\gamma$.}
    \label{fig:adult_gamma}
\end{figure}

\subsection{Evaluation on Multiple Environmental Shifts}\label{sec:multi_env_shifts}

\begin{figure}[H]
  \centering
  \begin{subfigure}{0.8\linewidth}
    \centering
    \includegraphics[width=\linewidth]{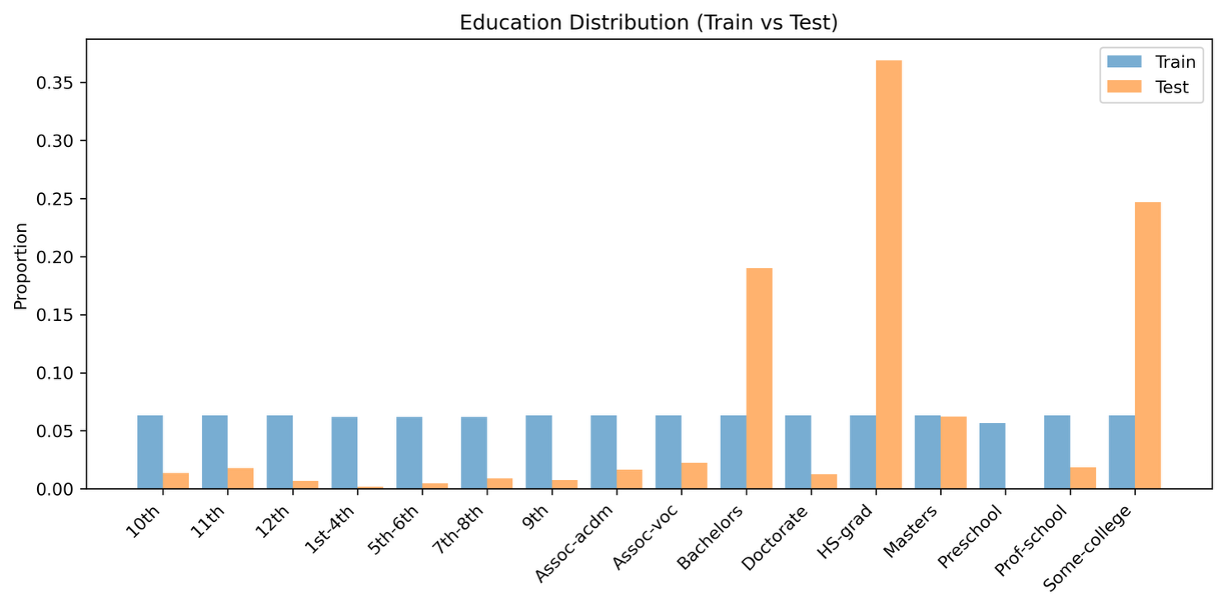}
    \caption{Test environment with 90\% above threshold and 10\% below threshold.}
    \label{fig:educ-90-10}
  \end{subfigure}

  \begin{subfigure}{0.8\linewidth}
    \centering
    \includegraphics[width=\linewidth]{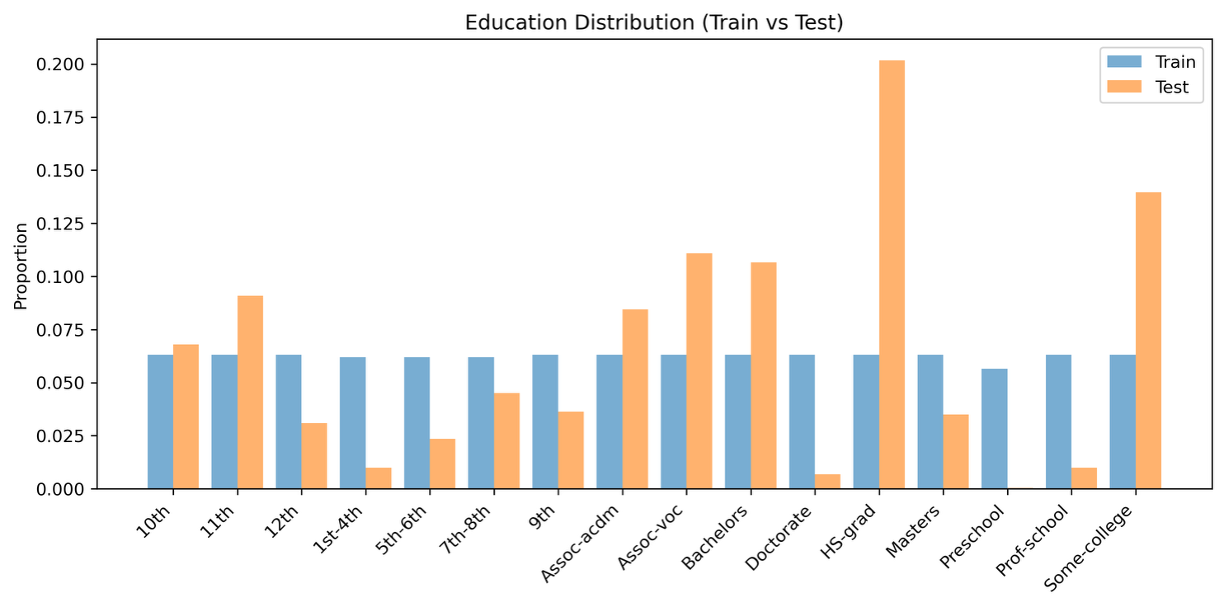}
    \caption{Balanced test environment with 50\%–50\% split.}
    \label{fig:educ-50-50}
  \end{subfigure}

  \begin{subfigure}{0.8\linewidth}
    \centering
    \includegraphics[width=\linewidth]{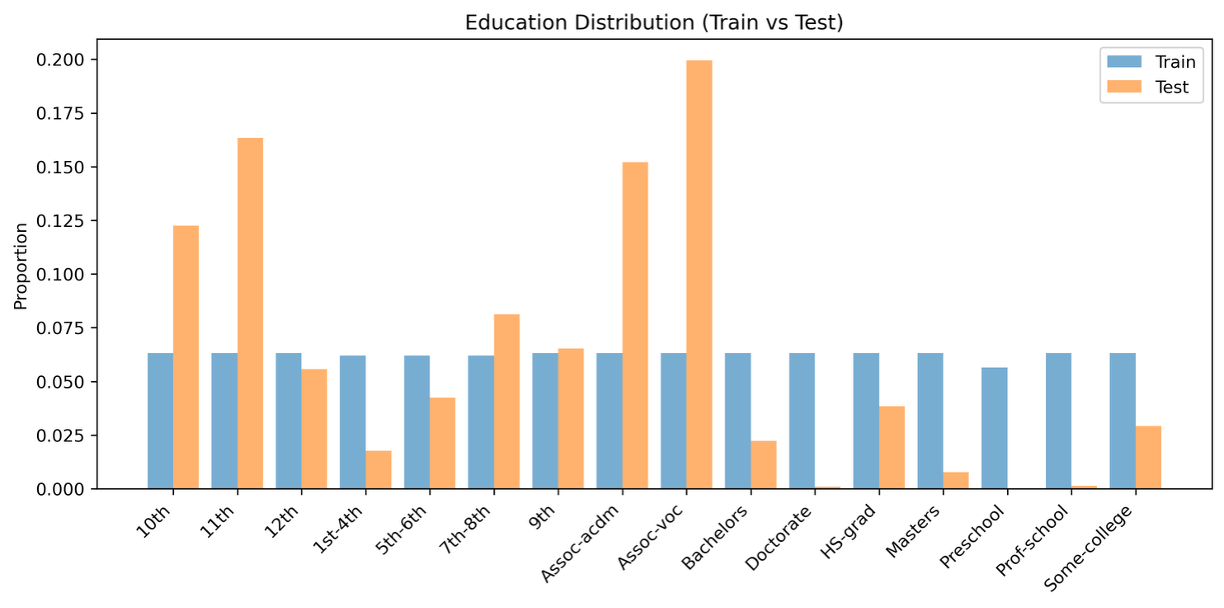}
    \caption{Test environment with 10\% above threshold and 90\% below threshold.}
    \label{fig:educ-10-90}
  \end{subfigure}

  \caption{Examples of constructed test environments by varying the proportion of samples above vs.\ below the education threshold. Shown are the extreme settings (90–10, 10–90) and the balanced case (50–50).}
  \label{fig:education-envs}
\end{figure}

In section \ref{sec:Income_prediction}, we considered a single controlled covariate shift between training and test distributions based on the \texttt{education} attribute. We now extend this analysis conducted on the Adult dataset by evaluating model performance across \emph{multiple test environments} in order to more thoroughly assess robustness.

The models are trained on the same training splits as described in Section~\ref{sec:Income_prediction}, where the marginal distribution of \texttt{education} is uniform. However, instead of a single test distribution, we construct a family of environments that vary systematically in their composition. Specifically, we use the encoded values of the \texttt{education} attribute\footnote{The encoding has been obtained via \texttt{StandardScaler} from \texttt{scikit-learn}.}. Based on a fixed threshold of $0.5$, we partition the population into two groups: those with \texttt{education} values below the threshold and those above. By varying the relative sizes of these two groups, we create environments with progressively different marginals.

In particular, we generate test distributions ranging from extremely imbalanced environments (90\% high vs.\ 10\% low and conversely 10\% high vs.\ 90\% low) to the balanced case (50\% high vs.\ 50\% low). This enables us to investigate not only average predictive performance, but also how stability is affected as the underlying distribution of education levels shifts dramatically. Figure~\ref{fig:education-envs} illustrates three representative examples of such constructed environments, depicted in terms of the \emph{decoded} education categories for interpretability.

This experimental design allows us to probe the extent to which each method is robust to unseen environmental changes, going beyond a single fixed train–test split. In practice, such variations are common in real-world deployment, where underlying demographics and education levels can vary substantially across regions or over time. 
\begin{figure}[t]
  \centering
  \includegraphics[width=\linewidth]{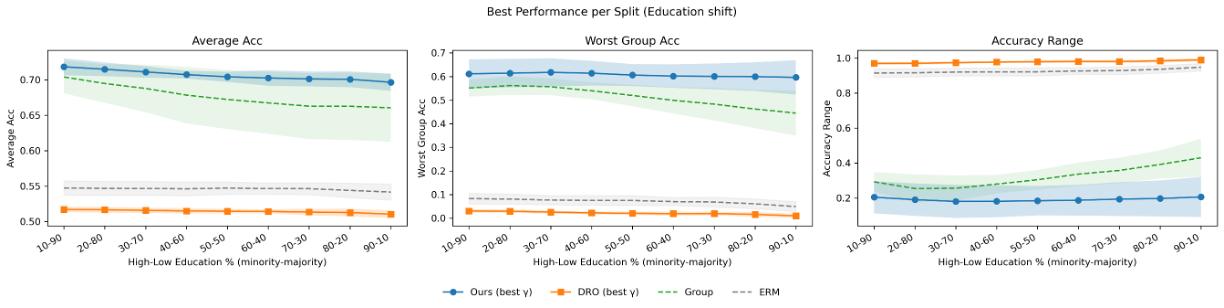}
  \caption{Best performance across education splits for all methods. 
  Models are trained on Adult with uniform training distribution over \texttt{education}. 
  Test environments are constructed with varying proportions of high- vs.\ low-education groups 
  (from $90$--$10$ to $10$--$90$). Plotted are the average accuracy, worst-group accuracy, 
  and accuracy range of each method at their best-performing $\gamma$ (if applicable).}
  \label{fig:adult-best-split}
\end{figure}

We report the best performance of each method across these splits in Figure~\ref{fig:adult-best-split}. Several consistent patterns emerge. First, ERM is unable to adapt to severe distributional imbalances: average accuracy remains near $0.55$ across all splits, worst-group accuracy collapses below $0.1$, and disparity across groups is consistently large ($\approx 0.9$). Standard DRO shows similar flaws, with even lower subgroup performance and accuracy ranges close to one in every environment. Group DRO improves worst-group accuracy substantially (around $0.40$–$0.45$) and narrows the group disparity. 

By contrast, our approach achieves the strongest and most stable performance across all environments. Average accuracy remains above $0.70$ even under extreme 90–10 splits, while worst-group accuracy is consistently high ($0.60$–$0.63$). Moreover, our method attains the lowest disparity, with accuracy ranges around $0.18$–$0.22$, nearly halving the gap compared to Group DRO and dramatically below DRO and ERM. Importantly, these advantages persist across the entire spectrum of test-set compositions, confirming that our method is robust not only to a single covariate shift but also to broad and heterogeneous environmental changes.

\section{Discussion}
The results across Sections~\ref{sec:Income_prediction}--\ref{sec:multi_env_shifts} highlight the importance of explicitly addressing both cross-group heterogeneity and distributional uncertainty in robust learning. In all tasks, Empirical Risk Minimization (ERM) is shown to be inadequate: it attains moderate average accuracy but collapses on minority subgroups, with worst-group accuracies close to zero and large disparities. This behavior reflects ERM’s tendency to overfit towards majority-dominated patterns, yielding models that perform poorly under imbalance or shift.

Standard DRO, while aiming to provide robustness, inherits similar limitations. Because a single ambiguity set is defined around the empirical distribution, the DRO objective remains dominated by well-sampled majority groups, offering little protection for rare or underrepresented subgroups. As a result, DRO performs even worse than ERM in some cases, with worst-group accuracies dropping to near-zero, and accuracy ranges approaching one. This emphasizes that classical distributional robustness, without group structure, cannot adequately safeguard vulnerable populations under domain shift.

Group DRO (GDRO) partially resolves these issues by focusing directly on the worst-performing group. In the experiment of Section \ref{sec:Income_prediction}, GDRO substantially improves minority-group performance and reduces disparities. However, these gains typically come at the expense of mean accuracy, since GDRO may overweight small and noisy groups. Moreover, as seen in the multiple-environment experiments, GDRO’s improvements are not consistent when covariate distributions vary more dramatically across test environments. The reliance on accurately estimated subgroup distributions makes it susceptible to uncontrolled shifts.

Our method achieves the strongest and most stable performance across all scenarios. On Adult with a single education-based covariate shift, it simultaneously maximizes average accuracy ($0.715$), worst-group accuracy ($0.613$), and minimizes disparities ($0.193$). Most notably, in the multi-environment Adult analysis, our method retains stable performance across a wide range of distributional shifts, maintaining accuracy above $0.70$, worst-group performance above $0.60$, and low disparity (below $0.22$) even under extreme 90--10 environment splits. These results demonstrate that the method not only exceeds Group DRO under imbalanced conditions, but also exhibits superior robustness under distributional shift, where its advantages become clearest.

A practical consideration in deployment is the choice of the robustness parameter $\gamma$ in our framework. Unlike the Wasserstein radius $\epsilon$ in classical DRO, which has a direct geometric interpretation, $\gamma$ serves as a cost penalty in the Lagrangian relaxation and lacks a physical meaning. Nonetheless, our experiments show that performance is remarkably stable across orders of magnitude of $\gamma$. For instance, in Section \ref{sec:Income_prediction}, the method maintains strong worst-group performance and small accuracy gaps across $\gamma \in [10^{-4}, 10^1]$. This insensitivity suggests that $\gamma$ need not be finely tuned, making the method easier to apply in practice. Conceptually, very small $\gamma$ values recover behavior closer to standard DRO, while very large $\gamma$ approximates Group DRO. Our results show that in between these two extremes, our method consistently finds solutions that balance fairness and robustness without compromising average accuracy.

Taken together, these findings underscore two key messages. First, robust learning must go beyond ERM and standard DRO to explicitly account for both group structure and intra-group uncertainty. Second, our proposed method provides a flexible and effective means of doing so: it interpolates between DRO and GDRO, achieves state-of-the-art trade-offs in worst-group and mean accuracy, and, crucially, shows its true potential when distributional shifts are most pronounced.

\section{Conclusion}
\label{sec:conclusion}
In this work, we introduced a novel framework for robust learning that jointly accounts for group-wise heterogeneity and within-group distributional uncertainty. Our formulation extends Group Distributionally Robust Optimization (Group DRO) by embedding local ambiguity sets within each group, thereby capturing uncertainty due to finite samples or shifts in local environments. We formulate the learning objective as a min-max-sup problem that optimizes over model parameters, adversarial distributions within each group, and group weights. We also present a tractable algorithm that alternates updates across these variables and provide convergence guarantees under standard assumptions. Empirical results on real-world dataset under multiple test settings demonstrate that existing methods addressing only one source of variability—either across groups or within groups—often fail to achieve robust and equitable performance. In contrast, our approach consistently improves worst-group accuracy and reduces disparities across subpopulations, emphasizing the importance of modeling both inter-group and intra-group uncertainty in a unified framework.

\addtolength{\textheight}{-11cm}
\bibliographystyle{apalike}
\bibliography{bibliography.bib}
\addtolength{\textheight}{+11cm}
\newpage
\appendix
\clearpage
\pagenumbering{arabic}
\setcounter{page}{1}
\setcounter{linenumber}{1} 
\linenumbers
\begin{center}
  {\LARGE\bfseries Appendix}
\end{center}

\section{Convergence Analysis}
\subsection{Preliminaries}
We first recall a few standard notions for smooth and weakly-convex functions on $\mathbb{R}^d$.

\begin{definition}[Lipschitz continuity]\label{def:Lipschitz}
A function $f:\mathbb{R}^d\to\mathbb{R}$ is \emph{$L$-Lipschitz} if for all $x,x'\in\mathbb{R}^d$,
\[
  \|f(x) - f(x')\| \;\le\; L\,\|x - x'\|.
\]
\end{definition}
\medskip

\begin{definition}[Smoothness]\label{def:smooth}
A differentiable function $f:\mathbb{R}^d\to\mathbb{R}$ is \emph{$\ell$-smooth} if for all $x,x'\in\mathbb{R}^d$,
\[
  \|\nabla f(x) - \nabla f(x')\| \;\le\; \ell\,\|x - x'\|.
\]
\end{definition}
\medskip
Consider the problem 
\[\min_x \max_y f(x, y).\]
Given this min-max problem we define
\[
  \Phi(x) \;=\;\max_{y\in Y} f(x,y),
\]
where $f(x,\cdot)$ is concave on a convex, bounded set $Y$.  Even though $\Phi$ may be nonconvex, one can still seek \emph{stationary points} of $\Phi$ as a proxy for global minimizers.

\begin{definition}[Stationarity --- differentiable case]\label{def:stat-diff}
A point $x\in\mathbb{R}^d$ is an \emph{$\varepsilon$-stationary point} of a differentiable $\Phi$ if
\[
  \|\nabla\Phi(x)\|\;\le\;\varepsilon.
\]
When $\varepsilon=0$, $x$ is a true stationary point.  
\end{definition}
\medskip

If $\Phi$ is not differentiable (e.g.\ in the general nonconvex-concave setting), we weaken this via \emph{weak convexity}.

\begin{definition}[Weak convexity]\label{def:weakly-convex}
A function $\Phi:\mathbb{R}^d\to\mathbb{R}$ is \emph{$\ell$-weakly convex} if
\[
  x\mapsto \Phi(x) + \tfrac{\ell}{2}\|x\|^2
\]
is convex.
\end{definition}
\medskip

In particular, one can define the \emph{Moreau envelope} of $\Phi$, which both smooths and regularizes it.

\begin{definition}[Moreau envelope]\label{def:Moreau}
For $\lambda>0$, the \emph{Moreau envelope} of $\Phi$ is
\[
  \Phi_\lambda(x)
  \;=\;\min_{w\in\mathbb{R}^d}\;\Bigl\{\Phi(w)\;+\;\tfrac{1}{2\lambda}\|w - x\|^2\Bigr\}.
\]
\end{definition}
\medskip
\begin{lemma}[Smoothness of the Moreau envelope]\label{lem:Moreau-smooth}
If $f$ is $\ell$-smooth and $Y$ is bounded, then the envelope $\Phi_{1/(2\ell)}$ of $\Phi(x)=\max_{y\in Y}f(x,y)$ is differentiable, $\ell$-smooth, and $\ell$-strongly convex.
\end{lemma}
\medskip

This allows an alternative stationarity measure:

\begin{definition}[Stationarity via Moreau envelope]\label{def:stat-Moreau}
A point $x$ is \emph{$\varepsilon$-stationary} for an $\ell$-weakly convex $\Phi$ if
\[
  \bigl\|\nabla \Phi_{1/(2\ell)}(x)\bigr\|\;\le\;\varepsilon.
\]
\end{definition}
\medskip

\begin{lemma}[Proximity to ordinary subgradients]\label{lem:stat-close}
If $x$ satisfies $\|\nabla \Phi_{1/(2\ell)}(x)\|\le\varepsilon$, then there exists $\hat x$ such that
\[
  \min_{\xi\in\partial\Phi(\hat x)}\|\xi\|\;\le\;\varepsilon
  \quad\text{and}\quad
  \|x-\hat x\|\;\le\;\tfrac{\varepsilon}{2\ell}.
\]
\end{lemma}

Finally, since our algorithm performs a mirror‐ascent update on the dual variable $q\in\Delta_G$, we require some standard facts about the associated Bregman divergence on the probability simplex.  Concretely, let 
\[
  \varphi\colon\mathbb{R}^G\to\mathbb{R}
  \quad\text{be a strictly convex $C^1$ generator.}
\]
Then for any $p_1,p_2\in\Delta^{n-1}$ the \emph{Bregman divergence} is defined by
\[
  D_{\varphi}(p_1\mathbin{\|}p_2)
  \;=\;\varphi(p_1)\;-\;\varphi(p_2)\;-\;\bigl\langle\nabla\varphi(p_2),\,p_1 - q\bigr\rangle.
\]

\begin{definition}[Legendre generator]
A function $\varphi$ is called a \emph{Legendre generator} on the simplex if
\begin{enumerate}
  \item $\varphi$ is strictly convex and continuously differentiable on the open simplex $\{x>0,\;\sum_i x_i=1\}$,
  \item its gradient $\nabla\varphi$ extends continuously to the closed simplex $\Delta^{n-1}$,
  \item and $\varphi$ attains its global minimum at the uniform distribution $u=(1/n,\dots,1/n)$.
\end{enumerate}
\end{definition}

\begin{example}[Negative‐entropy / KL generator]
When 
\[
  \varphi(x)=\sum_{i=1}^n x_i\ln x_i,
\]
one obtains the Kullback–Leibler divergence,
\[
  D_{\mathrm{KL}}(p_1\|p_2)
  = \sum_{i=1}^n p_{1,i}\ln\frac{p_{1,i}}{p_{2,i}}.
\]
Since we initialize and maintain all iterates in the interior of $\Delta_G$, this divergence remains finite throughout our mirror‐ascent steps.
\end{example}

\begin{property}[Nonnegativity \& convexity]
For any Legendre generator $\varphi$, one has
\[
  D_{\varphi}(p\|q)\;\ge0,
  \quad
  D_{\varphi}(p\|q)=0\iff p=q,
\]
and $D_{\varphi}(\,\cdot\,\|q)$ is convex in its first argument.
\end{property}

\begin{lemma}[Boundedness on the simplex]
If both $\varphi$ and $\nabla\varphi$ are bounded on the closed simplex, then
\[
  \max_{p_1,p_2\in\Delta^{n-1}} D_{\varphi}(p_1\|p_2)
  \;<\;\infty.
\]
\end{lemma}

\begin{lemma}[KL‐bound under interior iterates]
Suppose during mirror‐ascent every dual iterate 
\(q_t\in\Delta_G\) satisfies
\[
  q_{t,i}\;\ge\;\delta
  \quad\forall\,i=1,\dots,G,\;t=0,1,\dots,
\]
for some \(\delta>0\).  Then for any two such iterates \(p_1,p_2\in\Delta_G\),
\[
  D_{\mathrm{KL}}(p_1\|p_2)
  \;\le\; D.
\]
Where $D =ln \frac{1}{\delta}$.
\end{lemma}

\begin{proof}
Since \(p_{2,i}\ge\delta\) for all \(i\), we have
\[
  \ln\frac{p_{1,i}}{p_{2,i}}
  \;\le\;
  \ln\frac{p_{1,i}}{\delta}
  \;=\;
  \ln\frac{1}{\delta} \;+\; \ln p_{1,i},
\]
and because \(\sum_i p_{1,i}=1\) and \(\sum_i p_{1,i}\ln p_{1,i}\le0\),
\[
  D_{\mathrm{KL}}(p_1\|p_2)
  = \sum_i p_{1,i} \ln\frac{p_{1,i}}{p_{2,i}}
  \;\le\;
  \sum_i p_{1,i} \ln\frac{1}{\delta}
  \;+\;\sum_i p_{1,i}\ln p_{1,i}
  \;\le\;\ln\frac{1}{\delta}.
\]
\end{proof}

\subsection{Properties of the Robust Loss}
\begin{assumption}\label{assump:A}
The cost function $c:(\mathcal{X} \times \mathcal{Y} \times \mathcal{X}\times \mathcal{Y}) \rightarrow \mathbb{R}_+$ is continuous. For each $x \in \mathcal{X}$ and $y \in \mathcal{Y}$, $c(\cdot, (x,y))$ is 1-strongly convex with respect to the norm $||\cdot||$.
\end{assumption}

\begin{assumption}[Lipschitz Loss Function]
    \label{assump:C}
    Consider the loss function $\mathcal{L}$. Then for every function $f$ with model parameters $\theta \in \Theta$ and for every $(x,y) \in (\mathcal{X}, \mathcal{Y})$ we assume that $\mathcal{L}$ is K-Lipschitz with respect to $\theta$
\end{assumption}

\begin{assumption}[Lipschitz smoothness of the loss]\label{assump:B}
Consider the loss function $\mathcal{L}$. Then for every function $f$ with model parameters $\theta \in \Theta$ and for every $(x,y) \in (\mathcal{X}, \mathcal{Y})$ we assume
\begin{align*}
\bigl\|\nabla_{\theta}\,\mathcal{L}(f_\theta; x,y)\;-\;\nabla_{\theta}\,\mathcal{L}(f_{\theta'}; x,y)\bigr\|
&\le L_{\theta\theta}\,\|\theta-\theta'\|,\\
\bigl\|\nabla_{x,y}\,\mathcal{L}(f_\theta; x,y)\;-\;\nabla_{x,y}\,\mathcal{L}(f_\theta; x',y')\bigr\|
&\le L_{zz}\,\bigl\|\,(x,y)-(x',y')\bigr\|,\\
\bigl\|\nabla_{\theta}\,\mathcal{L}(f_\theta; x,y)\;-\;\nabla_{\theta}\,\mathcal{L}(f_\theta; x',y')\bigr\|
&\le L_{\theta z}\,\bigl\|\,(x,y)-(x',y')\bigr\|,\\
\bigl\|\nabla_{x,y}\,\mathcal{L}(f_\theta; x,y)\;-\;\nabla_{x,y}\,\mathcal{L}(f_{\theta'}; x,y)\bigr\|
&\le L_{z\theta}\,\|\theta-\theta'\|.
\end{align*}
\end{assumption}

\begin{lemma}[Smoothness of the penalized surrogate]\label{lem:smooth_phi}
Suppose the loss 
\(\mathcal{L}:\Theta\times(\mathcal X\times \mathcal Y)\to\mathbb{R}\)
satisfies Assumptions \ref{assump:A} and \ref{assump:B} (smoothness in \(\theta\) and \(z\)) with constants 
\(L_{\theta\theta},L_{\theta z},L_{z\theta},L_{zz}\), 
and that the transport cost \(c(z,w)\) is convex in \(w\).  Let $F_g$ denote the robust loss for each environment such that
\[F_g(\theta, x, y) = \sup_{(x^\prime, y^\prime)\in (\mathcal{X}, \mathcal{Y})} \phi(f_\theta; (x,y), (x^\prime, y^\prime)),\]
where
\[\phi(f_\theta; (x,y), (x^\prime, y^\prime)) = \mathcal{L}(f_\theta; x^\prime, y^\prime) - \gamma c((x,y), (x^\prime, y^\prime)).\]
Then we have that $F_g$ is $L_f-$smooth, where 
\[L_f = L_{\theta\theta} + \frac{L_{\theta z}L_{z \theta}}{[\gamma - L_{zz}]_+}.\]

\begin{proof}The proof follows from Lemma 1 in \cite{sinha2017certifying}
\end{proof}
\end{lemma}

\begin{lemma}[Lipschitzness of the robust surrogate]\label{lemma:lips}
Under Assumption \ref{assump:C}, the robust group‐level loss
\[
F_g(\theta, x, y)
=\sup_{(u,v)\in\mathcal X\times\mathcal Y}
\Bigl\{
\mathcal L\bigl(f_\theta;u,v\bigr)
-\gamma\,c\bigl((x,y),(u,v)\bigr)
\Bigr\}
\]
is $K$‐Lipschitz in~$\theta$.  
\end{lemma}

\begin{proof}
Let
\[
F_g(\theta; x,y)
=\sup_{(u,v)\in\mathcal X\times\mathcal Y}
\Bigl\{\,
\mathcal L\bigl(f_\theta;u,v\bigr)
-\gamma\,c\bigl((x,y),(u,v)\bigr)
\Bigr\}.
\]

\medskip\noindent
\textbf{(i) $F_g(\theta; x,y)-F_g(\theta'; x,y)\le K\,\|\theta-\theta'\|$.}

Choose 
\((u^*,v^*)\in\arg\max_{(u,v)}\{\mathcal L(f_\theta;u,v)-\gamma\,c((x,y),(u,v))\}\), which according to \cite{sinha2017certifying} exists for $\gamma > L_{zz}$. 
Then
\[
F_g(\theta; x,y)
=\mathcal L\bigl(f_\theta;u^*,v^*\bigr)
-\gamma\,c\bigl((x,y),(u^*,v^*)\bigr),
\]
and by definition of the supremum,
\[
F_g(\theta'; x,y)
\;\ge\;
\mathcal L\bigl(f_{\theta'};u^*,v^*\bigr)
-\gamma\,c\bigl((x,y),(u^*,v^*)\bigr).
\]
Subtracting gives
\[
F_g(\theta; x,y)-F_g(\theta'; x,y)
\;\le\;
\mathcal L\bigl(f_\theta;u^*,v^*\bigr)
-\mathcal L\bigl(f_{\theta'};u^*,v^*\bigr).
\]
By Assumption~\ref{assump:C},
\[
\bigl|\mathcal L(f_\theta;u^*,v^*)-\mathcal L(f_{\theta'};u^*,v^*)\bigr|
\;\le\;
K\,\|\theta-\theta'\|.
\]
Hence
\[
F_g(\theta; x,y)-F_g(\theta'; x,y)\;\le\;K\,\|\theta-\theta'\|.
\]

\medskip\noindent
\textbf{(ii) $F_g(\theta'; x,y)-F_g(\theta; x,y)\le K\,\|\theta-\theta'\|$.}

Choose 
\((\tilde u,\tilde v)\in\arg\max_{(u,v)}\{\mathcal L(f_{\theta'};u,v)-\gamma\,c((x,y),(u,v))\}\). Then
\[
F_g(\theta'; x,y)
=\mathcal L\bigl(f_{\theta'};\tilde u,\tilde v\bigr)
-\gamma\,c\bigl((x,y),(\tilde u,\tilde v)\bigr),
\]
and
\[
F_g(\theta; x,y)
\;\ge\;
\mathcal L\bigl(f_\theta;\tilde u,\tilde v\bigr)
-\gamma\,c\bigl((x,y),(\tilde u,\tilde v)\bigr).
\]
Subtracting yields
\[
F_g(\theta'; x,y)-F_g(\theta; x,y)
\;\le\;
\mathcal L\bigl(f_{\theta'};\tilde u,\tilde v\bigr)
-\mathcal L\bigl(f_\theta;\tilde u,\tilde v\bigr)
\;\le\;
K\,\|\theta-\theta'\|.
\]

\medskip\noindent
Combining (i) and (ii) gives
\[
\bigl|\,F_g(\theta; x,y)-F_g(\theta'; x,y)\bigr|
\;\le\;
K\,\|\theta-\theta'\|.
\]
\end{proof}

\subsection{Convergence Analysis for Descent–Mirror‐Ascent}
From Lemmas \ref{lem:smooth_phi} and \ref{lemma:lips} we have that the robust loss per group \[F_g(\theta) = \mathbb{E}_{(x,y)\sim \mathbb{P}^g_{X,Y}} \sup_{(u,v)\in (\mathcal{X}, \mathcal{Y})} \bigl\{ \mathcal{L}(f_\theta; u,v) - \gamma c((x,y), (u,v))\bigr\}\] is $L_f-$smooth and $K-$Lipschitz. Then we have the following lemma.
\begin{lemma}\label{lem:psi-smooth-lip}
Suppose for each group \(g=1,\dots,G\) the robust‐loss function
\[
  F_g(\theta)
  = \mathbb{E}_{(x,y)\sim\mathbb{P}^g_{X,Y}}
    \sup_{(u,v)\in\mathcal{X}\times\mathcal{Y}}
    \bigl\{\mathcal{L}(f_\theta;u,v) - \gamma\,c\bigl((x,y),(u,v)\bigr)\bigr\}
\]
is
\begin{itemize}
  \item \(\ell\)-smooth in \(\theta\): \(\|\nabla F_g(\theta)-\nabla F_g(\theta')\|\le\ell\,\|\theta-\theta'\|\), and
  \item \(K\)-Lipschitz in \(\theta\): \(\bigl|F_g(\theta)-F_g(\theta')\bigr|\le K\,\|\theta-\theta'\|\).
\end{itemize}
Define the weighted aggregate
\[
  \Psi(\theta,q)
  = \sum_{g=1}^G q_g\,F_g(\theta),
  \qquad
  (\theta,q)\in \Theta \times \Delta_G.
\]
Then:
\begin{enumerate}
  \item \(\Psi\) is \(\ell\)-smooth and 
  \item \(\Psi(\theta,\,\cdot\,)\) is \(K\)-Lipschitz in \(\theta\), uniformly over \(q\in\Delta_G\).
\end{enumerate}
\end{lemma}

\begin{proof}
Fix any \(q\in\Delta_G\) and \(\theta,\theta'\in\Theta\).  Since the weights \(q_g\ge0\) sum to \(1\), we have
\[
\begin{aligned}
\|\nabla_{\theta}\Psi(\theta,q)-\nabla_{\theta}\Psi(\theta',q)\|
&=\Bigl\|\sum_{g=1}^G q_g\bigl(\nabla F_g(\theta)-\nabla F_g(\theta')\bigr)\Bigr\|
\;\le\;
\sum_{g=1}^G q_g\,\ell\,\|\theta-\theta'\|
\\
&=\ell\,\|\theta-\theta'\|.
\end{aligned}
\]
Thus \(\Psi(\cdot,q)\) is \(\ell\)-smooth.  Similarly
\[
\begin{aligned}
\|\nabla_{q}\Psi(\theta,q)-\nabla_{q}\Psi(\theta,q')\|
&=0 \le 0\|q-q'\|.
\end{aligned}
\]
\(\Psi\) is \(\max\{0, \ell\} = \ell\)-smooth. 
Likewise
\[
\bigl|\Psi(\theta,q)-\Psi(\theta',q)\bigr|
=\Bigl|\sum_{g=1}^G q_g\bigl(F_g(\theta)-F_g(\theta')\bigr)\Bigr|
\;\le\;
\sum_{g=1}^G q_g\,K\,\|\theta-\theta'\|
=K\,\|\theta-\theta'\|.
\]
So \(\Psi(\cdot,q)\) is \(K\)-Lipschitz in \(\theta\).
\end{proof}

We define
\[
\Psi(\theta, q)\;=\;\sum_{g=1}^G q_g\,F_g(\theta).
\]
and from Lemma \ref{lem:psi-smooth-lip} we know that it is $L_f-$smooth and $K-$Lipschitz.

We also define \[P(\theta) = \max_{q \in \Delta_G} \Psi(\theta, q)\].
In order to prove convergence we use the notion of stationarity based on the Moreau envolope, such as $P_{1/2L_f}(\theta) = \min_{\theta^\prime} P(\theta^\prime)+L_f ||\theta^\prime - \theta||_2^2$. In this case, showing that the gradient of Moreau envolope converges to a small value is equal to showing that $\theta$ converges to a stationary point as shown in \cite{davis2019stochastic}.

\begin{lemma}\label{lem:first}
    Given assumptions \ref{assump:C} and \ref{assump:B}, let $\Delta_t = P(\theta_t) - \Psi(\theta_t, q_t)$.
    Then, we have 
    \[P_{1/2L_f}(\theta_t, ) \leq P_{1/2L_f}(\theta_{t-1}) + 2 \eta_\theta L_f \Delta_{t-1} - \frac{\eta_\theta}{4}||\nabla P_{1/2L_f}(\theta_{t-1})||^2 + \eta_\theta L_f K^2\]
    \begin{proof}
        The proof follows the same steps as the proof in the GDA version of Lemma D.3 in \cite{lin2020gradient}
    \end{proof}
\end{lemma}

\begin{lemma} \label{lem:second}
Given assumptions \ref{assump:C} and \ref{assump:B}, let $\Delta_t = P(\theta_t) - \Psi(\theta_t, q_t)$. Then $\forall s \le t-1$ we have 
\[\Delta_{t-1} \le \eta_\theta K^2 (2t - 2s - 1) + (\Psi(\theta_t, q_t) - \Psi(\theta_{t-1}, q_{t-1})) + L_f \bigl(D_{KL}(q^\star(\theta_s)||q_{t-1})  - D_{KL}(q^\star(\theta_s)||q_{t})\bigr)\]
where $q^\star(\theta) = arg\max_{q \in \Delta_G}\Psi(\theta, q)$.
\end{lemma}
\begin{proof}
    By the definition of Bregman Divergence we have that
    \[\eta_q (q - q_t)^T \nabla_q \Psi(\theta_{t-1}, q_{t-1}) \le D_{KL}(q||q_{t-1}) - D_{KL}(q||q_{t}) - D_{KL}(q_t||q_{t-1}).\]
    Since $\Psi(\theta_{t-1}, \cdot)$ is concave we have
    \[\Psi(\theta_{t-1}, q) \le \Psi(\theta_{t-1}, q_{t-1}) + (q - q_{t-1})\nabla_q \Psi(\theta_{t-1}, q_{t-1})\quad (1).\]
    Since $\Psi(\theta_{t-1}, \cdot)$ is $L_f-$smooth we have 
    \[-\Psi(\theta_{t-1}, q_t) \le -\Psi(\theta_{t-1}, q_{t-1}) - (q_t - q_{t-1})\nabla_q \Psi(\theta_{t-1}, q_{t-1}) +L_f D_{KL}(q_t||q_{t-1}) \quad (2).\]
    Adding $(1)$ and $(2)$, for $\eta_q = \frac{1}{L_f}$ and given the Bregman Definition, we get 
    \[
    \begin{aligned}
    \Psi\bigl(\theta_{t-1},\,q\bigr)\;-\;\Psi\bigl(\theta_{t-1},\,q_t\bigr)
    &\le (q - q_t)\;\nabla_q \Psi(\theta_{t-1},\,q_{t-1})
    \;+\;L_f\,D_{\mathrm{KL}}\bigl(q_t \,\|\, q_{t-1}\bigr)\\
    &\le L_f\,\biggl[ D_{\mathrm{KL}}\bigl(q \,\|\, q_{t-1}\bigr) - D_{\mathrm{KL}}\bigl(q \,\|\, q_{t}\bigr) \biggr]
    \end{aligned}
    \]
    Plugging $q = q^\star(\theta_s)$ for $s \le t-1$ we get
    \[
    \begin{aligned}
    \Psi\bigl(\theta_{t-1},\,q^\star(\theta_s)\bigr)\;-\;\Psi\bigl(\theta_{t-1},\,q_t\bigr)
    &\le L_f\,\biggl[ D_{\mathrm{KL}}\bigl(q^\star(\theta_s) \,\|\, q_{t-1}\bigr) - D_{\mathrm{KL}}\bigl(q^\star(\theta_s) \,\|\, q_{t}\bigr) \biggr]
    \end{aligned}
    \]
    By the definition of $\Delta_{t-1}$ we have
    \[
    \begin{aligned}
    \Delta_{t-1} \le{}&
    \bigl(\Psi(\theta_{t-1},\,q^\star(\theta_{t-1})) 
      \;-\;\Psi(\theta_{t-1},\,q^\star(\theta_s))\bigr)\\
    &\quad+\;\bigl(\Psi(\theta_t,\,q_t)
      \;-\;\Psi(\theta_{t-1},\,q_{t-1})\bigr)\\
    &\quad+\;L_f\,
      \Bigl[
        D_{\mathrm{KL}}\bigl(q^\star(\theta_s)\,\|\,q_{t-1}\bigr)
        \;-\;
        D_{\mathrm{KL}}\bigl(q^\star(\theta_s)\,\|\,q_t\bigr)
      \Bigr].
    \end{aligned}
    \]
    Since \(\Psi(\theta_s,\,q^\star(\theta_s)) \ge \Psi(\theta_s,\,q)\) for all \(q\in\Delta_G\), we obtain
\[
\begin{aligned}
\Psi\bigl(\theta_{t-1},\,q^*(\theta_{t-1})\bigr)
\;-\;\Psi\bigl(\theta_{t-1},\,q^*(\theta_s)\bigr)
&\;\le\;
\Bigl[\Psi\bigl(\theta_{t-1},\,q^*(\theta_{t-1})\bigr)
      -\Psi\bigl(\theta_s,\,q^*(\theta_{t-1})\bigr)\Bigr]\\
&\quad+\;\Bigl[\Psi\bigl(\theta_s,\,q^*(\theta_s)\bigr)
      -\Psi\bigl(\theta_{t-1},\,q^*(\theta_s)\bigr)\Bigr].
\end{aligned}
\]

Since \(\Psi(\cdot,q)\) is \(K\)–Lipschitz in \(\theta\) for any fixed \(q\), it follows that
\[
\begin{aligned}
\Psi\bigl(\theta_{t-1},\,q^*(\theta_{t-1})\bigr)
\;-\;
\Psi\bigl(\theta_s,\,q^*(\theta_{t-1})\bigr)
&\;\le\;
K\,\bigl\|\theta_{t-1}-\theta_s\bigr\|
\;\le\;
\eta_\theta\,K^2\,(t-1-s),\\
\Psi\bigl(\theta_s,\,q^*(\theta_s)\bigr)
\;-\;
\Psi\bigl(\theta_{t-1},\,q^*(\theta_s)\bigr)
&\;\le\;
K\,\bigl\|\theta_{t-1}-\theta_s\bigr\|
\;\le\;
\eta_\theta\,K^2\,(t-1-s),\\
\Psi\bigl(\theta_{t-1},\,q_t\bigr)
\;-\;
\Psi\bigl(\theta_t,\,q_t\bigr)
&\;\le\;
K\,\bigl\|\theta_{t-1}-\theta_t\bigr\|
\;\le\;
\eta_\theta\,K^2.
\end{aligned}
\]
Putting these pieces together yields the wanted result.
\end{proof}

\begin{lemma} \label{lem:third}
    Given assimptions \ref{assump:B} and \ref{assump:C} let $\Delta_t = P(\theta_t) - \Psi(\theta_t, q_t)$. Then the following statement holds
    \[\frac{1}{T+1}\bigl(\sum_{t=0}^T \Delta_t \bigr) \le \eta_\theta K^2 (B+1) + \frac{L_f D}{B} + \frac{\hat{\Delta}_0}{T+1}\]
    where $\hat{\Delta}_0 = P(\theta_0) - \Psi(\theta_0, q_0)$, $B$ is the block size of how we group the $\Delta_s \forall s \in [0, T]$ and where $D$ is the upper bound of the simplex $\Delta_G$ where $q$ takes values in using the Bregman Divergence.
\end{lemma}
\begin{proof}
    In the deterministic setting, we partition the sequence \(\{\Delta_t\}_{t=0}^T\) into blocks with size at most \(B\):
    \[
    \{\Delta_t\}_{t=0}^{B-1},\quad
    \{\Delta_t\}_{t=B}^{2B-1},\;\dots,\;
    \{\Delta_t\}_{t=T-B+1}^{T}.
    \]
    There are \(\bigl\lceil (T+1)/B \bigr\rceil\) such blocks.  Hence
    \[
    \frac{1}{T+1}\sum_{t=0}^T \Delta_t
    \;\le\;
    \frac{B}{T+1}
    \sum_{j=0}^{\bigl\lceil (T+1)/B\bigr\rceil -1}
    \biggl(
    \frac{1}{B}\sum_{t=jB}^{\min\{(j+1)B-1,\;T\}} \Delta_t
    \biggr).
    \]
    Furthermore, setting \(s=0\) in the inequality of Lemma \ref{lem:second} yields
    \[
    \begin{aligned}
    \sum_{t=0}^{B-1}\Delta_t
    &\;\le\;
    \eta_\theta\,K^2\,B^2
    \;+\; L_f \bigl(D_{KL}(q^\star(\theta_s)||q_{t-1})  - D_{KL}(q^\star(\theta_s)||q_{t})\bigr)
    \;+\;
    (\Psi(\theta_B, q_B) - \Psi(\theta_0, q_0))\\
    &\;\le\;
    \eta_\theta\,K^2\,B^2
    \;+\; L_f \;\ D
    \;+\;
    (\Psi(\theta_B, q_B) - \Psi(\theta_0, q_0)).
    \end{aligned}
    \]
Similarly, letting \(s=jB\) for \(1\le j\le \bigl\lceil (T+1)/B\bigr\rceil-1\) in Lemma~\ref{lem:second} gives
\[
\sum_{t=jB}^{(j+1)B-1}\Delta_t
\;\le\;
\eta_\theta\,K^2\,B^2
\;+\;
L_f\,D
\;+\;
\bigl[\Psi(\theta_{(j+1)B},\,q_{(j+1)B})
       \;-\;\Psi(\theta_{jB},\,q_{jB})\bigr].
\]
Combining the above gives
\[
\frac{1}{T+1}\sum_{t=0}^T \Delta_t
\;\le\;
\eta_\theta\,K^2\,B
\;+\;
\frac{L_f\,D}{B}
\;+\;
\frac{\Psi(\theta_{T+1},\,q_{T+1}) - \Psi(\theta_0,\,q_0)}{T+1}\,.
\]
Since \(\Psi(\cdot,q)\) is \(K\)-Lipschitz in \(\theta\) for any fixed \(q\), it follows that
\[
\begin{aligned}
\Psi\bigl(\theta_{T+1},\,q_{T+1}\bigr)\;-\;\Psi\bigl(\theta_0,\,q_0\bigr)
&=
\bigl[\Psi(\theta_{T+1},q_{T+1}) - \Psi(\theta_0,q_{T+1})\bigr]
\;+\;\bigl[\Psi(\theta_0,q_{T+1}) - \Psi(\theta_0,q_0)\bigr]\\
&\le
\eta_\theta\,K^2\,(T+1)\;+\;\widehat\Delta_0,
\end{aligned}
\]
where \(\widehat\Delta_0 = \Psi(\theta_0,q_{0}^\star) - \Psi(\theta_0,q_0)\).  

\end{proof}

\begin{theorem}[Convergence of Descent–Mirror‐Ascent]\label{thm:convergence}
Under Assumptions~\ref{assump:C} and~\ref{assump:B}, and choosing the step‐sizes
\[
\eta_\theta \;=\;\min\!\Bigl\{\tfrac{\varepsilon^2}{16L_fK^2},\,\tfrac{\varepsilon^4}{4096\,L_f^3K^2\,D}\Bigr\},
\qquad
\eta_q \;=\;\frac{1}{L_f},
\]
Algorithm~\ref{alg: minimaxfairauc} returns an $\varepsilon$‐stationary point of
\[
P(\theta) \;=\;\max_{q\in\Delta_G}\sum_{g=1}^G q_g\,F_g(\theta)
\]
in at most
\[
\mathcal{O}\!\Bigl(\frac{L_f^3K^2\,D\,\hat{\Delta}_P}{\varepsilon^6}
\;+\;\frac{L_f^3\,D\,\hat{\Delta}_0}{\varepsilon^4}\Bigr)
\]
iterations $\hat{\Delta}_0 = P(\theta_0) - \Psi(\theta_0, q_0)$ and \(\hat{\Delta}_P = P_{1/2L_f}(\theta_0) - \min_{\theta}P_{1/2L_f}(\theta)\).
\end{theorem}
\begin{proof}
    The proof follows the same steps to combine lemmas \ref{lem:first}, \ref{lem:second}, and \ref{lem:third} as in Theorem ... in \cite{sinha2017certifying}. The only difference is that we define $B$ as $B = \frac{D}{K}\sqrt{\frac{L_f}{\eta_\theta}}$.
\end{proof}

\end{document}